\documentclass{article} 
\usepackage{hyperref}
\usepackage{url}
\usepackage{amsthm,mathrsfs, amsmath, amsfonts,amssymb,comment,multirow}
\usepackage{graphicx}
\usepackage{algorithm,algorithmic,natbib}

\title{Regularization Effect of Fast Gradient Sign Method and its Generalization}

\author{
Chandler Zuo
\texttt{chandlerczuo@gmail.com} \\
}

\newtheorem{theorem}{Theorem}

\newtheorem{proposition}{Proposition}[theorem]

\begin{document}

\maketitle

\begin{abstract}

Fast Gradient Sign Method (FGSM) is a popular method to generate adversarial examples that make neural network models robust against perturbations. Despite its empirical success, its theoretical property is not well understood. This paper develops theory to explain the regularization effect of \textit{Generalized FGSM}, a class of methods to generate adversarial examples. Motivated from the relationship between FGSM and LASSO penalty, the asymptotic properties of Generalized FGSM are derived in the Generalized Linear Model setting, which is essentially the 1-layer neural network setting with certain activation functions. In such simple neural network models, I prove that Generalized FGSM estimation is $\sqrt n$-consistent and weakly oracle under proper conditions. The asymptotic results are also highly similar to penalized likelihood estimation. Compared to LASSO-type penalties, generalized FGSM introduces additional bias when data sampling is not \textit{sign neutral}, a concept I introduce to describe the balance-ness of the noise signs. Although the theory in this paper is developed under simple neural network settings, I argue that it may give insights and justification for FGSM in deep neural network settings as well.

\end{abstract}

\section{Introduction}\label{sec:introduction}

Fast Gradient Sign Method (FGSM) was introduced in \citep{goodfellow6572explaining} to improve the robustness of a neural network model against input purturbations. In a 1-layer neural network with sigmoid activation, suppose a sample of $n$ independent input-output pairs are denoted by  $(\mathbf x_i, y_i)_{i=1}^n$. FGSM maximizes the following objective function:

\begin{equation}\label{eq:logitobj}
Q_n(\mathbf \beta)=\sum_{i=1}^n\{y_i\tilde{\mathbf  x_i}^T\mathbf \beta-\log[1+\exp(\tilde{\mathbf x_i}^T\mathbf \beta)]\},
\end{equation}

where $\tilde{\mathbf x_i}=x_i+\eta_i$, and $\eta_i$ is the adversarial noise that maximally pertubes the loss function with $\mathscr L_\infty(p)$ norm not exceeding $\lambda_n$. Specifically, denote the loss function for a single observation before perturbation by $l$, i.e.

\begin{equation}
l(\mathbf \beta;\textbf x, y)=y\mathbf x^T\mathbf \beta-\log(1+\exp(\mathbf x^T\mathbf \beta),
\end{equation}

The maximal perturbation is

\[
\eta_i=\arg\min_{\eta:\|\eta\|_\infty\leq \lambda_n}l(\mathbf \beta;\mathbf x_i + \mathbf \eta, y_i).
\]

\citep{goodfellow6572explaining} already shows that the above $\eta_i=-\lambda_n \text{sign}[\triangledown{x_i}l(\beta;x,y)]$. Moreover, by substituting the expression of $\mathbf \eta_i$, eqn. (\ref{eq:logitobj}) can be rewritten as

\begin{equation}\label{eq:logitobj1}
Q_n(\mathbf \beta)=\sum_{i=1}^n\{y_i \mathbf x_i^T\mathbf \beta-y_i\lambda_n\text{sign}(e_i)\|\beta\|_1-\log[1+\exp(\mathbf x_i^T\mathbf \beta-\lambda_n\text{sign}(e_i)\|\beta\|_1)]\},
\end{equation}

where $e_i=y_i-[1+\exp(\mathbf x_i^T\mathbf \beta)]^{-1}$ is the residual in the regression model.

The effectiveness of FGSM has been discussed in various empirical studies (c.f. \citep{kurakin2016adversarial}, \citep{kurakin2016adversarial1}, \citep{goodfellow2018making}, \citep{tramer2017ensemble}, \citep{gilmer2018adversarial}). It also falls in the more general framework of adversarial examples training; see \citep{yuan2017adversarial} for a review of this area. To my best knowledge, however, there is no well established theory that explains the effectiveness of using adversarial examples. In the above literature, researchers have discussed and/or debated different aspects of adversarial attacks such as transferability, regularization and applicability to deep/shallow networks, but mostly relying on numerical experiments. Given the wide success of adversarial examples training in modern neural network training, it is thus desirable to establish mathematical theory to generate insights.

As an effort to derive such theory, in this paper, I focus on the regularization effect of FGSM and develop large sample results for a class of methods analogous to FGSM. To simplify the problem, the theory is developed for 1-layer neural network models. Such models are essentially regression models with certain activation functions. Asymptotic theory of such models is often feasible thanks to decades of development in statistical research. Results in this simple setting, nevertheless, may yield insights in behaviors of deep neural network models. One such example is the research in LASSO penalty. Its established asymptotic properties (c.f. \citep{knight2000asymptotics}, \citep{zhao2006model}, \citep{van2008high}) well explain its behaviors in deep neural network models. Despite the 1-layer restriction in my analysis, the activation function and the perturbation mechanism of FGSM are generalized. The results are thus extended beyond FGSM but also to a fairly large number of ways to generate adversarial examples and a large set of loss functions based on Generalized Linear Models.

This paper is organized as the following. Section \ref{sec:theory} introduces \textit{Generalized FGSM estimation}, a class of adversarial example generation schemas that extend the idea of FGSM through a class of penalty functions. Asymptotic theory is developed in Section \ref{sec:asymth}, which shows that the estimators are not only $\sqrt n$-consistent (Theorem \ref{th:convrate}) but also similar to the popular penalized likelihood estimation methods (Theorem \ref{th:penest}). Comparing to penalized likelihood estimation, Generalized FGSM automatically scales the penalty multipler by the noise level. It contains additional bias if the sampling space is not \textit{sign neutral}, a concept I introduce to describe the balanceness of error signs. In Section \ref{sec:examples}, I analyze two most popular models, linear regression and logistic regression, and show the sufficient conditions for covariate sampling that ensures the developed asymptotic property. A summary for the results and discussions of this paper is contained in Section \ref{sec:disc}.

\section{Generalized FGSM Estimation}\label{sec:theory}

\subsection{Model Set-up}

The 1-layer neural network model is analyzed in the following set-up. Let $(\mathbf x_i, y_i)_{i=1}^n$ be a sample of $n$ independent observations from $(\textbf x, Y)$ in the Generalized Linear Model (\citep{mccullagh1989generalized}) linking a $p$-dimensional predictor vector $\textbf x$ to a scalar response variable $Y$. The GLM assumes that with a canonical link function $b(\cdot)$, the conditional distribution for $Y$ given $\textbf x$ is the following distribution in the exponential family:

\[
f(y;\theta,\phi)=\exp\{y\theta-b(\theta)+c(y,\phi)\},
\]

where $\theta=\textbf x^T \mathbf \beta$. A specific case $b: \theta \mapsto \log(1+\exp(\theta))$ yields the logistic regression model that is equivalent to the sigmoid activation. The \textit{Generalized FGSM estimation} maximizes eqn. (\ref{eq:obj}), where I abuse the notation $Q_n$ in eqn. (\ref{eq:logitobj1}):

\begin{equation}\label{eq:obj}
Q_n(\mathbf \beta)=\sum_{i=1}^n[y_i \mathbf x_i^T\mathbf \beta-y_i\text{sign}(e_i)p_{\lambda_n}(\beta)-b(\mathbf x_i^T\mathbf \beta-\text{sign}(e_i)p_{\lambda_n}(\beta))],
\end{equation}

where $e_i=y_i-b'(\mathbf x_i^T\beta)$ and $p_{\lambda}(\cdot)$ is a penalty function defined on both $\mathbb R^p$ and $\mathbb R^1$ with the relation $p_{\lambda}(\theta)=\sum_{j=1}^pp_\lambda(\theta_j)$. One popular choice for this penalty function is the $\mathscr L_\gamma$ penalty $p_{\lambda}(\beta)=\lambda \|\beta\|_\gamma^\gamma$, with $\gamma>0$. Another popular choice is the SCAD penalty (\cite{fan2001variable}) defined through its continuous differentiable penalty:

\begin{equation}\label{eq:scad}
p_\lambda'(\theta)=\lambda\sum_{j=1}^p\text{sign}(\theta_j)[1\{|\theta_j|\leq \lambda)+\frac{(a\lambda-|\theta_j|)_+}{(a-1)\lambda}1\{|\theta_j|>\lambda\}]~~~\text{for some }a>2.
\end{equation}

Ad-hocly, eqn. (\ref{eq:obj}) defines a class of ways to generate adversarial examples. For each observed predictor $\mathbf x_i$, its perturbed counterpart is:

\begin{equation}\label{eq:perturb}
\tilde{\mathbf x_i}=\textbf x_i-\text{sign}(y_i-b(\mathbf x_i^T\beta))(1\{\beta_1\neq 0\}\frac{p_{\lambda_n}(\beta_1)}{\beta_1} , \cdots, 1\{\beta_p\neq 0\}\frac{p_{\lambda_n}(\beta_p)}{\beta_p})^T,
\end{equation}

where I define $0\cdot\frac{0}{0}=0$.

\subsection{Asymptotic Theory}\label{sec:asymth}

In the following, Theorem \ref{th:convrate} is the main result for the convergence rate for estimators for general penalty functions. Theorem \ref{th:penest} presents the weak limit of estimators for specific penalty functions.

\begin{theorem}\label{th:convrate}
(Convergence Rate)

Let $r_n=\frac{1}{\sqrt n}$ and $\epsilon_i=y_i-b'(\mathbf x_i^T\beta_0)$. Assume:

\begin{enumerate}
\item Suppose $(\mathbf x_i, y_i)$ are i.i.d. samples with $Eb''(\mathbf x_1^T\beta_0)\mathbf x_1\mathbf x_1^T=M\in \mathbb R^{p\times p}$.
\item 
\begin{equation}\label{eq:penrate}
\begin{aligned}
\alpha_n=&[\max_jp'_{\lambda_n}(\beta_{0j})1\{\beta_{0j}\neq 0\}]\vee [r_n^{-1} p_{\lambda_n}(r_nu)]=O(r_n),~~\forall u;\\
\tau_n=&\max_jp_{\lambda_n}(\beta_{0j})=O(r_n).\\
\end{aligned}
\end{equation}
\item $\forall C>0$,
\begin{equation}\label{eq:nearzero}
\sup_{\|\mathbf u\|=C}E[\epsilon_1(1\{0\leq \epsilon_1\leq r_n\mathbf x_1^T\mathbf u\}-1\{r_n\mathbf x_1^T\mathbf u\leq \epsilon_1\leq 0\})]=o(r_n).
\end{equation}
\begin{equation}\label{eq:sign}
\sup_{\|\mathbf u\|=C}E[\mathbf x_1^T\mathbf u1\{0\leq \epsilon_1\leq r_n\mathbf x_1^T\mathbf u\}-1\{r_n\mathbf x_1^T\mathbf u\leq \epsilon_1\leq 0\})]=o(1).
\end{equation}
\item $b(\cdot)$ is 3-rd order differentiable and $b''(\cdot)$ is bounded from above uniformly.
\end{enumerate}

Then there is a local maximizer $\hat\beta_n$ for $Q_n$ such that $\sqrt n(\hat\beta_n-\beta_0)=O_p(1)$.
\end{theorem}

\begin{proof}

Let $r_n$ be the convergence rate to be determined. If we can prove that $\forall \epsilon>0,~\exists C$ s.t.

\begin{equation}\label{eq:diffball}
P(\sup_{\|\mathbf u\|=C}Q_n(\beta_0+r_n\mathbf u)<Q_n(\beta_0))\geq 1-\epsilon,
\end{equation}

then there is a local maximum in the ball $\{\beta_0+r_n\mathbf u:\|\mathbf u\|\leq C\}$. Let 

\begin{equation}\label{eq:eqD}
\begin{aligned}
D_n(\mathbf u)=&Q_n (\beta_0+r_n\mathbf u)-Q_n(\beta_0)\\
=&\sum_{i=1}^n[ r_ny_i\mathbf x_i^T\mathbf u-y_i\text{sign}(e_i)p_{\lambda_n}(\beta_0+r_n\mathbf u)+y_i\text{sign}(\epsilon_i)p_{\lambda_n}(\beta_0)\\
& - b(\mathbf x_i^T\beta_0+r_n\mathbf x_i^T\mathbf u-\text{sign}(e_i)p_{\lambda_n}(\beta_0+r_n\mathbf u))\\
& + b(\mathbf x_i^T\beta_0-\text{sign}(\epsilon_i)p_{\lambda_n}(\beta_0))].\\
\end{aligned}
\end{equation}

By Taylor's expansion,

\begin{equation}\label{eq:taylorD}
\begin{aligned}
b(\mathbf x_i^T&\beta_0+r_n\mathbf x_i^T\mathbf u-\text{sign}(e_i)p_{\lambda_n}(\beta_0+r_n\mathbf u))-b(\mathbf x_i^T\beta_0-\text{sign}(\epsilon_i)p_{\lambda_n}(\beta_0))\\
=&b'(\mathbf x_i^T\beta_0)[r_n\mathbf x_i^T\mathbf u-\text{sign}(e_i)p_{\lambda_n}(\beta_0+r_n\mathbf u)]\\
&+\frac{1}{2}b''(\mathbf x_i^T\beta_0)[r_n\mathbf x_i^T\mathbf u-\text{sign}(e_i)p_{\lambda_n}(\beta_0+r_n\mathbf u)]^2(1+o(1))\\
&+b'(\mathbf x_i^T\beta_0)\text{sign}(\epsilon_i)p_{\lambda_n}(\beta_0)-\frac{1}{2}b''(\mathbf x_i^T\beta_0)p_{\lambda_n}(\beta_0)^2(1+o(1))
\end{aligned}
\end{equation}

Combining eqn. (\ref{eq:eqD}) and (\ref{eq:taylorD}), it can be seen that $D_n(\mathbf u)=D_{n,1}(\mathbf u)+D_{n,2}(\mathbf u)+D_{n,3}(\mathbf u)+D_{n,4}(\mathbf u)$, where

\begin{equation}\label{eq:components}
\begin{aligned}
D_{n,1}(\mathbf u)=&\sum_{i=1}^n[r_n\epsilon_i\mathbf x_i^T\mathbf u-\frac{1}{2}r_n^2b''(\mathbf x_i^T\beta_0)\mathbf u^T\mathbf x_i\mathbf x_i^T\mathbf u(1+o_p(1))],\\
D_{n,2}(\mathbf u)=&\sum_{i=1}^n\epsilon_i[\text{sign}(\epsilon_i)p_{\lambda_n}(\beta_0)-\text{sign}(e_i)p_{\lambda_n}(\beta_0+r_n\mathbf u)],\\
D_{n,3}(\mathbf u)=&\sum_{i=1}^nb''(\mathbf x_i^T\beta_0)r_n\mathbf x_i^T\mathbf u\text{sign}(e_i)p_{\lambda_n}(\beta_0+r_n\mathbf u)(1+o(1)),\\
D_{n,4}(\mathbf u)=&\frac{1}{2}\sum_{i=1}^nb''(\mathbf x_i^T\beta_0)[p_{\lambda_n}(\beta_0)^2(1+o_p(1))-p_{\lambda_n}(\beta_0+r_n\mathbf u)^2(1+o(1))].\\
\end{aligned}
\end{equation}

By the moment conditions of $b''(\mathbf x_1^T\beta_0)\mathbf x_1^T\mathbf x_1$, Law of Large Numbers and Central Limit Theorem are applicable, which yield:

\[
\frac{1}{\sqrt n}\sum_{i=1}^n\epsilon_i\mathbf x_i\overset{p}\rightarrow W\sim N(0, M),~\frac{1}{n}\sum_{i=1}^nb''(\mathbf x_i^T\beta_0)\mathbf x_i\mathbf x_i^T\overset{p}\rightarrow M.
\]

Here, the covariance matrix for $W$ is computed by:

\[
E(\epsilon_1^2\mathbf x_1^T\mathbf x_1)=E[E(\epsilon_1^2\mathbf x_1^T\mathbf x_1|\mathbf x_1)]=E(b''(\mathbf x_1^T\beta_0)\mathbf x_1^T\mathbf x_1)=M.
\]

Therefore, $D_{n,1}(\mathbf u)=\sqrt n r_nW\mathbf u-\frac{1}{2}nr_n^2\mathbf u^TM\mathbf u+o_p(nr_n^2)$.

\[
\begin{aligned}
D_{n,2}(\mathbf u)=&\sum_{i=1}^n\epsilon_i\text{sign}(\epsilon_i)[p_{\lambda_n}(\beta_0)-p_{\lambda_n}(\beta_0+r_n\mathbf u)]\\
&+\sum_{i=1}^n\epsilon_ip_{\lambda_n}(\beta_0+r_n\mathbf u)[\text{sign}(\epsilon_i)-\text{sign}(e_i)]\\
=&D_{n,2}^1(\mathbf u)+D_{n,2}^2(\mathbf u).\\
\end{aligned}
\]

\[
\begin{aligned}
D_{n,2}^1(\mathbf u)=&nE|\epsilon_1|[p_{\lambda_n}(\beta_0)-p_{\lambda_n}(\beta_0+r_n\mathbf u)](1+o_p(1))\\
=&-nE|\epsilon_1|(1+o_p(1))\cdot\\
&\sum_{j=1}^p[p'_{\lambda_n}(\beta_{0j})r_nu_j1\{\beta_{0j}\neq 0\}(1+o(1))+p_{\lambda_n}(r_nu_j)1\{\beta_{0j}=0\}]\\
=&-nE|\epsilon_1|\sum_{j=1}^p[p'_{\lambda_n}(\beta_{0j})r_nu_j1\{\beta_{0j}\neq 0\}+p_{\lambda_n}(r_nu_j)1\{\beta_{0j}=0\}]\\
&+o_p(n\alpha_nr_n).
\end{aligned}
\]

The above lines also show that $p_{\lambda_n}(\beta_0)-p_{\lambda_n}(\beta_0+r_n\mathbf u)=O_p(\alpha_nr_n)$. Using this as well as the definition of $\tau_n$,

\begin{equation}\label{eq:penaltyrate}
\begin{aligned}
p_{\lambda_n}(\beta_0+r_n\mathbf u)=&p_{\lambda_n}(\beta_0)+p_{\lambda_n}(\beta_0+r_n\mathbf u)-p_{\lambda_n}(\beta_0)\\
=&O_p(\tau_n+\alpha_nr_n).\\
\end{aligned}
\end{equation}

Using eqn. (\ref{eq:nearzero}),

\[
\begin{aligned}
|D_{n,2}^2|=&O_p(\tau_n+\alpha_nr_n)\sum_{i=1}^n\epsilon_i[\text{sign}(\epsilon_i)-\text{sign}(e_i)]\\
=&O_p(\tau_n+\alpha_nr_n)\sum_{i=1}^n|\epsilon_i|1\{\text{sign}(\epsilon_i)\neq\text{sign}(\epsilon_i-r_n\mathbf x_i^T\mathbf u)\}\\
=&O_p(n(\tau_n+\alpha_nr_n))E[\epsilon_11\{0\leq \epsilon_1\leq r_n\mathbf x_1^Tu\}-\epsilon_11\{r_n\mathbf x_1^Tu\leq \epsilon_1\leq 0\}]\\
=&o_p(nr_n(\tau_n+\alpha_nr_n)).\\
\end{aligned}
\]

Therefore,

\begin{equation}\label{eq:Dn2}
\begin{aligned}
D_{n,2}(\mathbf u)=&-nE|\epsilon_1|\sum_{j=1}^p[p'_{\lambda_n}(\beta_{0j})r_nu_j1\{\beta_{0j}\neq 0\}+p_{\lambda_n}(r_nu_j)1\{\beta_{0j}=0\}]\\
&+o_p(n\alpha_nr_n^2\vee n\alpha_nr_n\vee n\tau_nr_n).
\end{aligned}
\end{equation}

For $D_{n,3}$, using eqn. (\ref{eq:penaltyrate}) and (\ref{eq:sign}),

\[
\begin{aligned}
D_{n,3}(\mathbf u)=&\sum_{i=1}^nb''(\mathbf x_i^T\beta_0)r_n\mathbf x_i^T\mathbf u\text{sign}(\epsilon_i)p_{\lambda_n}(\beta_0+r_n\mathbf u)\\
&+\sum_{i=1}^nb''(\mathbf x_i^T\beta_0)r_n\mathbf x_i^T\mathbf u[\text{sign}(\epsilon_i-r_n\mathbf x_i^T\mathbf u)-\text{sign}(\epsilon_i)]p_{\lambda_n}(\beta_0+r_n\mathbf u)\\
=&r_np_{\lambda_n}(\beta_0+r_n\mathbf u)[\sum_{i=1}^nb''(\mathbf x_i^T\beta_0)\text{sign}(\epsilon_i)\mathbf x_i]^T\mathbf u\\
&+r_np_{\lambda_n}(\beta_0+r_n\mathbf u)\sum_{i=1}^nb''(\mathbf x_i^T\beta_0)\mathbf x_i^T\mathbf u[1\{r_n\mathbf x_i^T\mathbf u\leq \epsilon_i\leq 0\}-1\{0\leq \epsilon_i\leq r_n\mathbf x_i^T\mathbf u\}]\\
=&nr_np_{\lambda_n}(\beta_0+r_n\mathbf u)[Eb''(\mathbf x_1^T\beta_0)\text{sign}(\epsilon_1)\mathbf x_1]^T\mathbf u(1+o_p(1))\\
&+nr_np_{\lambda_n}(\beta_0+r_n\mathbf u)\\
&\cdot E[b''(\mathbf x_1^T\beta_0)\mathbf x_1^T\mathbf u(1\{r_n\mathbf x_1^T\mathbf u\leq \epsilon_1\leq 0\}-1\{0\leq \epsilon_1\leq r_n\mathbf x_1^T\mathbf u\})](1+o_p(1))\\
=&nr_np_{\lambda_n}(\beta_0+r_n\mathbf u)[Eb''(\mathbf x_1^T\beta_0)\text{sign}(\epsilon_1)\mathbf x_1]^T\mathbf u+o_p(nr_n(\tau_n+\alpha_nr_n))\\
=&nr_np_{\lambda_n}(\beta_0)[Eb''(\mathbf x_1^T\beta_0)\text{sign}(\epsilon_1)\mathbf x_1]^T\mathbf u\\
&+O_p(n\alpha_nr_n^2)+o_p(nr_n(\tau_n+\alpha_nr_n)).\\
\end{aligned}
\]

\[
\begin{aligned}
D_{n,4}(\mathbf u)=&O_p(\sum_{i=1}^n[p_{\lambda_n}(\beta_0)^2(1+o_p(1))-p_{\lambda_n}(\beta_0+r_n\mathbf u)^2(1+o_p(1))])\\
=&O_p(n\sum_{j=1}^p[2p_{\lambda_n}(\beta_{0j})p'_{\lambda_n}(\beta_{0j})r_nu_j1\{\beta_{0j}\neq 0\}+p_{\lambda_n}(r_nu_j)^21\{\beta_{0j}=0\}])\\
&+o_p(np_{\lambda_n}(\beta_0)^2+np_{\lambda_n}(\beta_0+r_n\mathbf u)^2)\\
=&O_p(n\alpha_nr_n(\tau_n+\alpha_nr_n))+o_p(n(\tau_n^2\vee \alpha_n^2r_n^2))\\
=&O_p(n\alpha_nr_n(\tau_n+\alpha_nr_n))+o_p(n\tau_n^2).
\end{aligned}
\]

As a result,

\[
\begin{aligned}
D_n(\mathbf u)=&\sqrt n r_nW\mathbf u-\frac{1}{2}nr_n^2\mathbf u^TM\mathbf u\\
&-nE|\epsilon_1|\sum_{j=1}^p[p'_{\lambda_n}(\beta_{0j})r_nu_j1\{\beta_{0j}\neq 0\}+p_{\lambda_n}(r_nu_j)1\{\beta_{0j}=0\}]\\
&+nr_np_{\lambda_n}(\beta_0)[Eb''(\mathbf x_1^T\beta_0)\text{sign}(\epsilon_1)\mathbf x_1]^T\mathbf u\\
&+O_p(n\alpha_nr_n\tau_n+n\alpha_nr_n^2+n\alpha_n^2r_n^2)\\
&+o_p(n\alpha_nr_n+nr_n^2+n\alpha_nr_n^2+n\tau_nr_n+n\tau_n^2).
\end{aligned}
\]

At this point, we can let $\sqrt n r_n=1$. Given eqn. (\ref{eq:penrate}),

\begin{equation}\label{eq:limitD}
\begin{aligned}
D_n(\mathbf u)=&W\mathbf u-\frac{1}{2}\mathbf u^TM\mathbf u\\
&-nE|\epsilon_1|\sum_{j=1}^p[p'_{\lambda_n}(\beta_{0j})\frac{u_j}{\sqrt n}1\{\beta_{0j}\neq 0\}+p_{\lambda_n}(\frac{u_j}{\sqrt n})1\{\beta_{0j}=0\}]\\
&+\sqrt np_{\lambda_n}(\beta_0)[Eb''(\mathbf x_1^T\beta_0)\text{sign}(\epsilon_1)\mathbf x_1]^T\mathbf u\\
&+o_p(1),\\
\end{aligned}
\end{equation}

where the third term is bounded by $\sqrt n\alpha_nE|\epsilon_1|\|\mathbf u\|$ and the fourth term is bounded by $\sqrt n\tau_n\|Eb''(\mathbf x_1^T\beta_0)\text{sign}(\epsilon_1)\mathbf x_1\|\|\mathbf u\|$. Therefore, with sufficiently large $C$, the second term in eqn. (\ref{eq:limitD}) dominates all other terms, and eqn. (\ref{eq:diffball}) holds.

\end{proof}

The following theorems gives the right rate of penalty for various penalty functions.

\begin{theorem}\label{th:penest}
(Weak Limit)

Assume all conditions in Theorem \ref{th:convrate} except eqn. (\ref{eq:penrate}). Also, consider $\hat\beta_n=\arg\max_{\|\beta-\beta_0\|\leq \frac{K}{\sqrt n}} Q_n(\beta)$. Denote $V=E[b''(\mathbf x_1^T\beta_0)\text{sign}(\epsilon_1)\mathbf x_1]$.

\begin{enumerate}
\item For $p_{\lambda_n}(\beta)=\lambda_n\|\beta\|_\gamma^\gamma$ with $\gamma>1$. If $\sqrt n\lambda_n\rightarrow \lambda_0$, then eqn. (\ref{eq:penrate}) holds. Moreover, $\sqrt n(\hat \beta_n-\beta_0)\overset{d}\rightarrow \arg\max D(\mathbf u)$, where

\[
D(\mathbf u)=W\mathbf u-\frac{1}{2}\mathbf u^TM\mathbf u-\gamma\lambda_0E|\epsilon_1|\sum_{j=1}^p\text{sign}(\beta_{0j})|\beta_{0j}|^{\gamma-1}u_j+\lambda_0\|\beta_0\|^\gamma_\gamma V^T\mathbf u.
\]

\item For LASSO penalty $p_{\lambda_n}(\beta)=\lambda_n\|\beta\|_1$. If $\sqrt n\lambda_n\rightarrow \lambda_0$, then eqn. (\ref{eq:penrate}) holds. Moreover, $\sqrt n(\hat \beta_n-\beta_0)\overset{d}\rightarrow \arg\max D(\mathbf u)$, where

\[
\begin{aligned}
D(\mathbf u)=&W\mathbf u-\frac{1}{2}\mathbf u^TM\mathbf u+\lambda_0\|\beta_0\|_1 V^T\mathbf u\\
&-\lambda_0E|\epsilon_1|\sum_{j=1}^p[\text{sign}(\beta_{0j})1\{\beta_{0j}\neq 0\}u_j+1\{\beta_{0j}=0\}|u_j|].\\
\end{aligned}
\]

\item For $p_{\lambda_n}(\beta)=\lambda_n\|\beta\|_\gamma^\gamma$ with $0<\gamma<1$. If $n^{1-\frac{\gamma}{2}}\lambda_n\rightarrow \lambda_0$, then eqn. (\ref{eq:penrate}) holds. Moreover, if $D(\mathbf u)$ has a unique maximum in $\{\mathbf u: \|\mathbf u\|\leq K\}$, where

\[
D(\mathbf u)=W\mathbf u-\frac{1}{2}\mathbf u^TM\mathbf u-\lambda_0E|\epsilon_1|\sum_{j=1}^p1\{\beta_{0j}=0\}|u_j|^\gamma+\lambda_0\|\beta_0\|^\gamma_\gamma V^T\mathbf u.
\]

Then $\sqrt n(\hat \beta_n-\beta_0)\overset{d}\rightarrow \arg\max D(\mathbf u)$.

\item For $p_{\lambda_n}$ be the SCAD penalty defined as eqn. (\ref{eq:scad}). If $\sqrt n\lambda_n\rightarrow \lambda_0$, then eqn. (\ref{eq:penrate}) holds. Moreover, $\sqrt n(\hat \beta_n-\beta_0)\overset{d}\rightarrow \arg\max D(\mathbf u)$, where

\[
D(\mathbf u)=W\mathbf u-\frac{1}{2}\mathbf u^TM\mathbf u-\lambda_0E|\epsilon_1|\sum_{j=1}^p1\{\beta_{0j}=0\}|u_j|.
\]

\end{enumerate}

\end{theorem}

\begin{proof}

Let $\mathbf u_n=\sqrt n(\hat \beta_n-\beta_0)$.

First, $D_n(\mathbf u)\overset{p}\rightarrow D(\mathbf u)$ for each type of penalty functions, where $D_n$ is defined in eqn. (\ref{eq:limitD}). This can be seen from the steps in the proof of Theorem \ref{th:convrate}. Moreover, from that proof it can also be seen that the convergence is uniform on the compact set $\{\mathbf u:\|\mathbf u\|\leq K\}$. 

Next, $D(\mathbf u)$ has a unique maximum in all cases. For the cases with SCAD or $\mathscr L_\gamma$ with $\gamma\geq 1$, this results from $D(\mathbf u)$'s convexity. For $\mathscr L_\gamma$ with $0<\gamma <1$, this is assumed in the theorem's statement.

By Theorem 14.1 in \citep{kosorok2008introduction}, $\mathbf u_n\overset{d}{\rightarrow} \arg\max D(\mathbf u)$.

\end{proof}

\begin{theorem}\label{th:oracle}
(Weak Oracle Property)

Consider $\mathscr L_\gamma$ penalty and SCAD penalty. Let $\beta_0=(\beta_{01},\beta_{02})^T$ where $\beta_{02}=\mathbf 0$. Partition the estimator in Theorem \ref{th:penest} $\hat\beta_n=(\hat\beta_{n1},\hat\beta_{n2})$ accordingly. Also partition $M$, $V$ and $W$ as $W=(W_1^T, W_2^T)^T$, $V=(V_1^T, V_2^T)^T$, and 

\[
M=
\begin{bmatrix}
M_{11} & M_{12}\\
M_{21} & M_{22}\\
\end{bmatrix}
.
\]

With the same conditions as in Theorem \ref{th:penest}, $\forall \epsilon > 0$, $\exists \lambda_0$ s.t. $P\{\hat\beta_{n1}=\tilde \beta,\hat\beta_{n2}=\mathbf 0\} \geq 1-\epsilon$, where.

\[
\tilde \beta=\left\{
\begin{array}{ll}
M_{11}^{-1}(W_1+\lambda_0\|\beta_0\|_\gamma^\gamma V_1),&\mathscr L_\gamma\text{ penalty};\\
M_{11}^{-1}W_1,&\text{SCAD penalty}.\\
\end{array}
\right.
\]

\end{theorem}

\begin{proof}
For $\mathscr L_\gamma$ penalty, since $\mathbf u_n\overset{d}\rightarrow \arg\max D(\mathbf u)$, it suffices to show that $\forall \epsilon > 0$, $\exists \lambda_0$ s.t.

\begin{equation}\label{eq:limitoracle}
P\{\arg\max_{\|\mathbf u\|\leq K}D(\mathbf u)=(\tilde \beta^T, \mathbf 0^T)^T\}\geq 1-\epsilon.
\end{equation}

Since the conditions of Theorem \ref{th:penest} already guarantees that maximum is unique, it suffices to check first order conditions. To this end, let $s$ be the dimension of $\beta_{01}$, and $\tilde u=(\tilde \beta^T, \mathbf 0^T)^T$. For $1\leq j\leq s$,

\begin{equation}\label{eq:diffsparse}
\frac{\partial D(\mathbf u)}{\mathbf u}|_{\mathbf u=\tilde u,j}=[-M\tilde u+W+\lambda_0\|\beta_0\|_\gamma^\gamma V]_j=0.
\end{equation}

When $j>s$,

\[
\frac{\partial D(\mathbf u)}{\mathbf u}|_{\mathbf u=\tilde u,j}=\left\{
\begin{array}{lr}
[-M\tilde u+W+\lambda_0\|\beta_0\|_\gamma^\gamma V]_j-\lambda_0\gamma E|\epsilon_1||u_j|^{\gamma-1}, & u_j>0;\\
{[-M\tilde u+W+\lambda_0\|\beta_0\|_\gamma^\gamma V]}_j+\lambda_0\gamma E|\epsilon_1||u_j|^{\gamma-1}, & u_j<0.\\
\end{array}\right.
\]

It can be seen that ${[-M\tilde u+W+\lambda_0\|\beta_0\|_\gamma^\gamma V]}_j=O_p(1)$. Therefore, choose sufficiently large $\lambda_0$, with probability $1-\epsilon$, when $u_j$ is close to 0,

\begin{equation}\label{eq:subdiff}
\frac{\partial D(\mathbf u)}{\mathbf u}|_{\mathbf u=\tilde u,j}=\left\{
\begin{array}{lr}
<0, & u_j>0;\\
>0, & u_j<0.\\
\end{array}\right.
\end{equation}

For SCAD penalty, when $j>s$,

\[
\frac{\partial D(\mathbf u)}{\mathbf u}|_{\mathbf u=\tilde u,j}=\left\{
\begin{array}{lr}
[-M\tilde u+W]_j-\lambda_0E|\epsilon_1|, & u_j>0;\\
{[-M\tilde u+W]}_j+\lambda_0E|\epsilon_1|, & u_j<0.\\
\end{array}\right.
\]

Eqn. (\ref{eq:diffsparse}) and (\ref{eq:subdiff}) together prove  (\ref{eq:limitoracle}) for both penalty functions.

\end{proof}

\subsection{Comparison with Penalized Likelihood Estimation}

It is worth comparing the theoretical results in this section with the widely applied penalized likelihood estimation in the following form:

\[
\tilde Q_n(\mathbf \beta)=\sum_{i=1}^n[y_i \mathbf x_i^T\mathbf \beta-b(\mathbf x_i^T\mathbf \beta)]-np_{\lambda_n}(\beta).
\]

For $\mathscr L_\gamma$ penalty, results in Theorem \ref{th:penest} is almost identical to Theorem 2 and 3 in \cite{knight2000asymptotics}. Following the notation of this paper, \cite{knight2000asymptotics} shows that, under the settings of Theorem \ref{th:penest}, the asymptotic limit of the maximized function is:

\[
\tilde D(\mathbf u)=W\mathbf u-\frac{1}{2}\mathbf u^TM\mathbf u-\left\{
\begin{array}{lr}
\gamma\lambda_0\sum_{j=1}^p\text{sign}(\beta_{0j})|\beta_{0j}|^{\gamma-1}u_j,&\gamma > 1;\\
\lambda_0\sum_{j=1}^p[\text{sign}(\beta_{0j})u_j+|u_j|1\{\beta_{0j}=0\}],&\gamma=1;\\
\lambda_0\sum_{j=1}^p|u_j|^\gamma1\{\beta_{0j}=0\},&0<\gamma<1.\\
\end{array}
\right.
\]

There are two key differences between $\tilde D$ and $D$. The first difference is the multipler for the penalty; $\lambda_0$ in $\tilde D$ is replaced by $\lambda_0E|\epsilon_1|$. This means that, in Generalized FGSM, the effective strength of the penalty is automatically scaled by the noise level $E|\epsilon_1|$. Mathematically, such scaling originates from the fact that the penalty function $p_{\lambda_n}$ is multipled by the noises $\epsilon_i$ in eqn. (\ref{eq:components}). Choosing the penalty level to be proportional to the noise level has been shown to have theoretical advantage (\cite{sun2012scaled}, \cite{dalalyan2017prediction}). In this respect, such automatic scaling is an advantage of Generalized FGSM.

The second difference is that $D$ introduces additional bias through $V$. To explain this quantity, I introduce the definition of \textit{sign neutral}, which means $V=0$. In certain cases, such as for linear regression models, sign neutrality can be achieved quite trivially (see Section \ref{sec:lr}). However, this is non-trivial for general cases when $\epsilon_1$ and $\mathbf x_1$ are not independent, as in the logistic regression case (see Section \ref{sec:logit}). Heuristically, $V=0$ requires that signs of the errors across the sample space even out. When errors are dependent on the covariates, this actually requires that the sampling for $\mathbf x_i$'s to be balanced in some way. If such sampling is not balanced, Generalized FGSM introduces additional bias. Notice, however, whether such bias is beneficial or malicious requires analysis for the generalization bound that is not covered in this paper.

Theorem \ref{th:oracle} is weaker than the oracle results for SCAD penalty in \cite{fan2001variable} in that the estimator in this paper cannot be achieved with probability converging to 1. The oracle result in \cite{fan2001variable} requires $\sqrt n \lambda_n\rightarrow \infty$, which makes $\alpha_n\wedge \tau_n=\Omega(n^{-\frac{1}{2}})$. This would invalidate the results in Theorem \ref{th:convrate} since it would not ensure that $D_{n,4}$ in eqn. (\ref{eq:components}) is of order $o_p(1)$. 

\section{Examples}\label{sec:examples}

This section discusses the applicability of theory in Section \ref{sec:theory} to two popular GLM models: linear regression and logistic regression. The key point of my discussion is analyzing requirements to satisfy eqn. (\ref{eq:nearzero}) and (\ref{eq:sign}). All other conditions for Theorem \ref{th:convrate} can be satisfied with standard assumptions. 

\subsection{Linear Regression}\label{sec:lr}

Consider the linear regression model $y_i=\mathbf x_i^T\beta_0+\epsilon_i$ with i.i.d. Gaussian errors $\epsilon_i\sim N(0,\sigma^2)$. The following theorem shows that conditions eqn. (\ref{eq:nearzero}) and (\ref{eq:sign}) are trivially satisfied.

\begin{theorem}\label{th:lr}
For the linear regression model, suppose $\mathbf x_1, \cdots, \mathbf x_n, \epsilon_1, \cdots, \epsilon_n$ are i.i.d. with $E\mathbf x_1\mathbf x_1^T=M\in \mathbb R^{p\times p}$.

\begin{enumerate}
\item If eqn. (\ref{eq:penrate}) holds, then the conclusions of Theorem \ref{th:convrate} hold.
\item For specific penalty functions, assuming the corresponding rate for $\lambda_n$ as in Theorem \ref{th:penest}. Then the conclusions of Theorem \ref{th:penest} and \ref{th:oracle} hold.
\end{enumerate}

\end{theorem}

\begin{proof}

$E\epsilon_1^2\mathbf x_1\mathbf x_1^T=\sigma^2E\mathbf x_1\mathbf x_1^T=M<\infty$ by the assumption. $b(\theta)=\theta^2/2$ is also 3rd order differentiable. It only remains to verify eqn. (\ref{eq:nearzero}) and (\ref{eq:sign}).

Let $\phi(\cdot)$ and $\Phi(\cdot)$ be the c.d.f. and p.d.f. of the standard Gaussian distribution. Then,

\[
\begin{aligned}
&|E[\epsilon_1(1\{0\leq \epsilon_1\leq r_n\mathbf x_1^T\mathbf u\}-1\{r_n\mathbf x_1^T\mathbf u\leq \epsilon_1\leq 0\})]|\\
\leq&E[|\epsilon_1|(1\{0\leq \epsilon_1\leq r_n\mathbf x_1^T\mathbf u\}+1\{r_n\mathbf x_1^T\mathbf u\leq \epsilon_1\leq 0\})]\\
\leq & E[|\epsilon_1|1\{|\epsilon_1|\leq r_n|\mathbf x_1^T\mathbf u|\}]=E[2\int_0^{r_n|\mathbf x_1^T\mathbf u|}\frac{x}{\sigma}\phi(\frac{x}{\sigma})dx]\\
\leq&E[\int_0^{r_n|\mathbf x_1^T\mathbf u|}\frac{2x}{\sigma}\phi(0)dx]=\frac{\phi(0)}{\sigma}r_n^2E|\mathbf x_1^T\mathbf u|=\frac{\phi(0)}{\sigma}r_n^2u^TMu\\
=&O(r_n^2).\\
\end{aligned}
\]

This proves eqn. (\ref{eq:nearzero}). Similarly,

\[
\begin{aligned}
&|E[\mathbf x_1^T\mathbf u(1\{0\leq \epsilon_1\leq r_n\mathbf x_1^T\mathbf u\}-1\{r_n\mathbf x_1^T\mathbf u\leq \epsilon_1\leq 0\})]|\\
\leq & E[|\mathbf x_1^T\mathbf u|1\{|\epsilon_1|\leq r_n|\mathbf x_1^T\mathbf u|\}]\leq r_n\frac{\phi(0)}{\sigma}E[(\mathbf x_1^T\mathbf u)^2]\\
= &r_n\frac{\phi(0)}{\sigma}\mathbf u^TM\mathbf u=O(r_n).\\
\end{aligned}
\]

Therefore, eqn. (\ref{eq:sign}) is proved.

\end{proof}

\subsection{Logistic Regression}\label{sec:logit}

\begin{theorem}\label{th:logit}
For the logistic regression model, suppose 

\begin{enumerate}
\item $(\mathbf x_i, y_i)$ are i.i.d. observations with $y_i|\mathbf x_i\sim Bin([1+\exp(\mathbf x_i^T\beta_0)]^{-1})$;
\item $E[\|\mathbf x_1\|(1+e^{|\mathbf x_1^T\beta_0|})]<\infty$, $E[\|\mathbf x_1\|^2(1+e^{|\mathbf x_1^T\beta_0|})]<\infty$.
\end{enumerate}

Then:

\begin{enumerate}
\item If eqn. (\ref{eq:penrate}) holds, then the conclusions of Theorem \ref{th:convrate} hold.
\item For specific penalty functions, assuming the corresponding rate for $\lambda_n$ as in Theorem \ref{th:penest}. Then the conclusions of Theorem \ref{th:penest} and \ref{th:oracle} hold.
\end{enumerate}

\end{theorem}

\begin{proof}

For logistic regression, the canonical link function is $b(\theta)=\log(1+e^\theta)$, which is 3rd order differentiable, and $b''(\theta)\leq \frac{1}{4}$. Also, $\forall \mathbf u$, $\mathbf u^TEb''(\mathbf x_1^T\beta_0)\mathbf x_1\mathbf x_1^T\mathbf u=Eb''(\mathbf x_1^T\beta_0)(\mathbf x_1^T\mathbf u)^2\leq \frac{1}{4} E(\mathbf x_1^T\mathbf u)^2\leq \frac{1}{4}\|\mathbf u\|^2E\|\mathbf x_1\|^2<\frac{1}{4}\|\mathbf u\|^2E[\|\mathbf x_1\|^2(1+e^{|\mathbf x_1^T\beta_0|})<\infty$. This shows that $E[b''(\mathbf x_1^T\beta_0)\mathbf x_1\mathbf x_1^T]=M$ exists. It remains to prove eqn. (\ref{eq:nearzero}) and (\ref{eq:sign}).

Denote $p=\frac{1}{1+e^{-\theta}}$.

\[
\begin{aligned}
&|E[\epsilon_1(1\{0\leq \epsilon_1\leq r_n\mathbf x_1^T\mathbf u\}-1\{r_n\mathbf x_1^T\mathbf u\leq \epsilon_1\leq 0\})]|\\
\leq & E[|\epsilon_1|1\{|\epsilon_1|\leq r_n|\mathbf x_1^T\mathbf u|\}]\\
=&E[p(1-p)\cdot 1\{\frac{1}{1+e^{\mathbf x_1^T\beta_0}}\leq r_n|\mathbf x_1^T\mathbf u|\}+(1-p)p\cdot 1\{\frac{1}{1+e^{-\mathbf x_1^T\beta_0}}\leq r_n|\mathbf x_1^T\mathbf u|\}]\\
\leq &\frac{1}{4}E[1\{\frac{1}{1+e^{\mathbf x_1^T\beta_0}}\leq r_n|\mathbf x_1^T\mathbf u|\}+1\{\frac{1}{1+e^{-\mathbf x_1^T\beta_0}}\leq r_n|\mathbf x_1^T\mathbf u|\}]\\
\leq &\frac{1}{2}P\{\frac{1}{1+e^{\mathbf x_1^T\beta_0}}\wedge \frac{1}{1+e^{-\mathbf x_1^T\beta_0}}\leq r_n|\mathbf x_1^T\mathbf u|\}=\frac{1}{2}P\{\frac{1}{1+e^{|\mathbf x_1^T\beta_0|}}\leq r_n|\mathbf x_1^T\mathbf u|\}\\
\leq&\frac{1}{2}P\{\|\mathbf x_1\|(1+e^{|\mathbf x_1^T\beta_0|})\geq\frac{1}{r_n\|\mathbf u\|}\}\leq\frac{r_n\|\mathbf u\|}{2}E[\|\mathbf x_1\|(1+e^{|\mathbf x_1^T\beta_0|})].\\
\end{aligned}
\]

This proves eqn. (\ref{eq:nearzero}).

\[
\begin{aligned}
&|E[\mathbf x_1^T\mathbf u(1\{0\leq \epsilon_1\leq r_n\mathbf x_1^T\mathbf u\}-1\{r_n\mathbf x_1^T\mathbf u\leq \epsilon_1\leq 0\})]|\\
\leq & E[|\mathbf x_1^T\mathbf u|1\{|\epsilon_1|\leq r_n|\mathbf x_1^T\mathbf u|\}]\\
=&E\{|\mathbf x_1^T\mathbf u|[p1\{1-p\leq r_n|\mathbf x_1^T\mathbf u|\}+(1-p)1\{p\leq r_n|\mathbf x_1^T\mathbf u|\}]\}\\
\leq&2E(|\mathbf x_1^T\mathbf u|1\{p\wedge(1-p)\leq r_n|\mathbf x_1^T\mathbf u|\})=2E(|\mathbf x_1^T\mathbf u|1\{\frac{1}{1+e^{|\mathbf x_1^T\beta_0|}}\leq r_n|\mathbf x_1^T\mathbf u|\})\\
\leq &2r_nE(\mathbf x_1^T\mathbf u)^2(1+e^{|\mathbf x_1^T\beta_0|})\leq 2r_n\|u\|^2E[\|\mathbf x_1\|^2(1+e^{|\mathbf x_1^T\beta_0|})]=O(r_n).\\
\end{aligned}
\]

This proves eqn. (\ref{eq:sign}).

\end{proof}

It is worth noting that the second condition of Theorem \ref{th:logit} requires strong bound for all moments of $\mathbf x_1$, which suggests that the distribution of $\mathbf x_1$ need be light tailed. Such a strong condition is not needed for linear regression models as in Theorem \ref{th:lr}, and is indeed an artefact from the dependence between $\epsilon_1$ and $\mathbf x_1$. Nevertheless, a sufficient condition for the distribution of $\mathbf x_1$ is presented in Proposition \ref{pp:tail}. This condition requires that the distribution of the norm $\|\mathbf x_1\|$ has exponential tails, which is reasonable since it can be satisfied by a large number of multi-variate distributions such as Gaussian distribution, exponential distribution and any distribution with bounded support.

\begin{proposition}\label{pp:tail}
Denote the p.d.f. of $\mathbf x_1$ by $f_{\mathbf x_1}(\mathbf x)$ w.r.t. Lebesgue measure $\nu$ on $\mathbb R^p$. If $\exists C>0, A>\|\beta_0\|$ s.t. $f_{\mathbf x_1}(\mathbf x)\leq Ce^{-A\|\mathbf x\|}$. Then $E[\|\mathbf x_1\|^n(1+e^{|\mathbf x_1^T\beta_0|})]<\infty$ for all $n>0$.
\end{proposition}

\begin{proof}

By change of integration variable $t=\|\mathbf x\|$, we have

\[
\begin{aligned}
&E[\|\mathbf x_1\|^n(1+e^{|\mathbf x_1^T\beta_0})]\leq E[\|\mathbf x_1\|^n(1+e^{\|\mathbf x_1\|\|\beta_0\|})]\\
\leq &C\int_{\mathbb R^p}\|\mathbf x\|^n(1+e^{\|\mathbf x\|\|\beta_0\|})\cdot e^{-A\|\mathbf x\|}d\nu(\mathbf x)\\
=&C\int_0^\infty t^{n}(e^{-At}+e^{-(A-\|\beta_0\|)t})\cdot \frac{2\pi^{\frac{p+1}{2}}t^p}{\Gamma(\frac{p+1}{2})}dt<\infty.
\end{aligned}
\]

\end{proof}

\section{Conclusions and Discussion}\label{sec:disc}

FGSM is a popular modern technique in the area of adversarial examples learning. Despite its empirical success, its theoretical property is not well studied. In this paper, I have introduced the Generalized FGSM method that extends the relationship between FGSM and LASSO regression. Applying this method in the GLM framework, which is also the 1-layer neural network framework with certain activation functions, I have developed asymptotic theory that shows Generalized FGSM achieves $\sqrt n$-consistency, sparsity, and is weakly oracle.

The developed asymptotic theory is also analogous to penalized likelihood estimation methods. Comparing to those methods, Generalized FGSM has the advantage that the penalty multipler is automatically scaled by the noise level. It also introduces additional bias if the sampling distribution is not sign neutral. For logistic regression models, Generalized FGSM performs the best when covariate sampling has light tails. For deeper neural network models, I conjecture that sampling balanceness is important for the optimal performance of FGSM. Validation of this statement can be of future work.

The objective of this paper is to develop theory in simple neural network settings that may bring theoretical justification for FGSM. For this purpose, I have not developed algorithms for Generalized FGSM estimation. Nevertheless, the class of penalty functions in this method induces a number of ways to generate adversarial examples (eqn. (\ref{eq:perturb})), making exsiting adversarial examples learning algorithms readily applicable. Evaluating the performance of these new adversarial example generating schemes can be future empirical study topics.

\bibliography{lasso_fgsm}
\end{document}